\documentclass[review]{elsarticle}

\usepackage{lineno,hyperref}
\modulolinenumbers[5]
\usepackage{graphicx} 
\usepackage{subfigure} 
\usepackage{amsmath}
\usepackage{amsfonts}
\usepackage{amsthm}
\usepackage{dsfont}
\usepackage{array}
\usepackage{color}
\usepackage{bm}
\usepackage{booktabs}
\usepackage{algorithm}
\usepackage{algorithmic}

\newcommand\given[1][]{\:#1\vert\:}

\newtheorem{theorem}{Theorem}
\newtheorem*{theorem*}{Theorem}

\journal{Pattern Recognition}

\begin{document}

\begin{frontmatter}

\title{Target contrastive pessimistic risk \\for robust domain adaptation}

\author[a]{Wouter M. Kouw\corref{ca}}
\cortext[ca]{Corresponding author}
\ead{W.M.Kouw@tudelft.nl}

\author[a,b]{Marco Loog}
\address[a]{Pattern Recognition Laboratory \\ Delft University of Technology \\ Mekelweg 4, 2628 CD Delft, The Netherlands}
\address[b]{The Image Group \\ University of Copenhagen \\ Universitetsparken 5, DK-2100 Copenhagen, Denmark}

\begin{abstract}
In domain adaptation, classifiers with information from a source domain adapt to generalize to a target domain. However, an adaptive classifier can perform worse than a non-adaptive classifier due to invalid assumptions, increased sensitivity to estimation errors or model misspecification. Our goal is to develop a domain-adaptive classifier that is robust in the sense that it does not rely on restrictive assumptions on how the source and target domains relate to each other and that it does not perform worse than the non-adaptive classifier. We formulate a conservative parameter estimator that only deviates from the source classifier when a lower risk is guaranteed for all possible labellings of the given target samples. We derive the classical least-squares and discriminant analysis cases and show that these perform on par with state-of-the-art domain adaptive classifiers in sample selection bias settings, while outperforming them in more general domain adaptation settings.
\end{abstract}


\end{frontmatter}

%\linenumbers

\newpage
\section{Introduction}
\label{intro}
Generalization in supervised learning relies on the fact that future samples should originate from the same underlying distribution as the ones used for training. However, this is not the case in settings where data is collected from different locations, different measurement instruments are used or there is only access to biased data. In these situations the labeled data does not represent the distribution of interest. This problem setting is referred to as a \emph{domain adaptation} setting, where the distribution of the labeled data is called the \emph{source domain} and the distribution that one is actually interested in is called the \emph{target domain}. Most often, data in the target domain is not labeled and adapting a source domain classifier, i.e. changing its predictions to be more suited to the target domain, is the only means by which one can make predictions for the target domain. Unfortunately, depending on the domain dissimilarity, adaptive classifiers can perform \emph{worse} than non-adaptive ones. In this work, we formulate a conservative adaptive classifier that always performs at least as well as the non-adaptive one.

Biased samplings tend to occur when one samples locally from a much larger population \cite{quionero2009dataset,moreno2012unifying}. For instance, in computer-assisted diagnosis, biometrics collected from two different hospitals will be different due to differences between the patient populations: ones diet might not be the same as the others. Nonetheless, both patient populations are subsamples of the human population as a whole. Adaptation in this example corresponds to accounting for the differences between patient populations, training a classifier on the corrected labeled data from one hospital, and applying the adapted classifier to the other hospital. Additionally, different measurement instruments cause different biased samplings: photos of the same object taken with different cameras lead to different distributions over images \cite{gong2012geodesic}. Lastly, biases arise when one only has access to particular subsets, such as data from individual humans in a activity recognition task \cite{hachiya2012importance}.

In the general setting, domains can be arbitrarily different and contain almost no mutual information, which means generalization will be extremely difficult. However, there are cases where the problem setting is more structured: in the \emph{covariate shift} setting, the marginal data distributions differ but the class-posterior distributions are equal \cite{shimodaira2000improving,cortes2008sample,bickel2009discriminative}. This means that the underlying true classification function is the same in both domains, implying that a correctly specified adaptive classifier converges to the same solution as the target classifier. Adaptation occurs by weighing each source sample by how important it is under the target distribution and training on the importance-weighed labeled source data. A model that relies on equal class-posterior distributions can perform very well when its assumption is true, but it can deviate in detrimental ways when its assumption is false. 

Considering their potential, a number of papers have looked at conditions and assumptions that allow for successful adaptation. A particular robust one specifies the existence of a common latent embedding, represented by a set of \emph{transfer components} \cite{pan2011domain}. After mapping data onto these components, one can train and test standard classsifiers again. Other possible assumptions include low-data-divergence \cite{ben2007analysis,ben2010impossibility,ben2010theory}, low-error joint prediction \cite{ben2010impossibility,ben2010theory}, the existence of a domain manifold \cite{gopalan2011domain,baktashmotlagh2014domain,patel2015visual}, restrictions to subspace transformations \cite{fernando2013unsupervised}, conditional independence of class and target given source data \cite{kouw2016feature} and unconfoundedness \cite{imbens2015causal}. The more restrictive an assumption is, the worse the classifier tends to perform when it is invalid. One of the strengths of the estimator that we develop in this paper is that it does not require making any assumptions on the relationship between the domains.

The domain adaptation and covariate shift settings are very similar to the sample selection bias setting in the statistics and econometrics communities \cite{heckman1977sample,zadrozny2004learning,bickel2009discriminative}. There, the bias is explicitly modeled as a variable that denotes how likely it is for a particular sample to be selected for the training set. One hopes to generalize to an unbiased sample, i.e., the case where each sample is equally likely to be selected. As such, this setting can also be viewed as a case of domain adaptation, with the biased sample set as the source domain and the unbiased sample set as the target domain. In this case, there is even additional information: the support of the source domain will be contained in the support of the target domain. This information can be exploited, as some methods rely on a non-zero target probability for every source sample \cite{cortes2008sample,gretton2009covariate}  Lastly, the causal inference community has also considered causes for differing training and testing distributions, including how to estimate and control for these differences \cite{storkey2009training,scholkopf2012causal,moreno2012unifying}.

\paragraph{}
Although not often discussed, a variety of papers have reported adaptive classifiers that, at times, perform worse than the non-adaptive source classifier \cite{cortes2008sample,gong2013connecting,baktashmotlagh2014domain,cortes2014domain,liu2014robust,kouw2016feature}. On closer inspection, this tends to happen when a classifier with a particular assumption is deployed in a problem setting for which this assumption is not valid. For example, if the assumption of a common latent representation does not hold or when the domains are too dissimilar to recover the transfer components, then mapping both source and target data onto the found transfer components will result in mixing of the class-conditional distributions \cite{pan2011domain}. Additionally, one of the most popular covariate shift approaches, kernel mean matching ({\sc kmm}), assumes that the support of the target distribution is contained in the support of the source distribution \cite{huang2007correcting,gretton2009covariate}. When this is not the case, the resulting estimated weights can become very bimodal: a few samples are given very large weights and all other samples are given near-zero weights. This greatly reduces the effective sample size for the subsequent classifier \cite{mcbook}. 

Since the validity of the aforementioned assumptions are difficult, if not impossible, to check, it is of interest to design an adaptive classifier that is at least guaranteed to perform as well as the non-adaptive one. Such a property is often framed as a minimax optimization problem in statistics, econometrics and game theory \cite{berger2013statistical}. Wen et al. constructed a minimax estimator for the covariate shift setting: Robust Covariate Shift Adjustment ({\sc rcsa}) \cite{wen2014robust} accounts for estimation errors in the importance weights by considering their worst-case configuration. However, this can sometimes be too conservative, as the worst-case weights can be very disruptive to the subsequent classifier optimization. Another minimax strategy, dubbed the Robust Bias-Aware ({\sc rba}) classifier \cite{liu2014robust}, plays a game between a risk minimizing target classifier and a risk maximizing target class-posterior distribution, where the adversary is constrained to pick posteriors that match the moments of the source distribution statistics. This constraint is important, as the adversary would otherwise be able to design posterior probabilities that result in degenerate classifiers (e.g. assign all class-posterior probabilities to $1$ for one class and $0$ for the other). However, it also means that their approach loses predictive power in areas of feature space where the source distribution has limited support, and thus is not suited very well for problems where the domains are very different. 

\paragraph{}
The main contribution of our paper is that we provide an empirical risk minimization framework to train a classifier that will always perform at least as well as the naive source classifier. Furthermore, we show that a discriminant analysis model derived from our framework will \emph{always be likelier} than the naive source model. To the best of our knowledge, strict improvements have not been shown before.

\paragraph{}
The paper continues as follows: section \ref{tcpr} presents the motivation and general formulation of our method, with the specific case of a least-squares classifier in section \ref{case_ls} and the specific case of a discriminant analysis classifier in section \ref{case_da}. Sections \ref{exp_ssb} and \ref{exp_da} show experiments on sample selection bias problems and general domain adaptation problems, respectively, and we conclude with discussing some limitations and implications in section \ref{disc}.

\section{Target Contrastive Pessimistic Risk}
\label{tcpr}
This section starts with the problem definition, followed by our risk formulation.

\subsection{Problem definition}
Given a sample space, a \emph{domain} refers to a particular probability measure over this sample space. One has access to labeled data from one domain, denoted the \emph{source} domain, and aims to generalize to another domain, denoted the \emph{target} domain, where no labels are available. Assuming that the labels follow a random variable ${\cal Y}$ taking values in the set $\{1, \dots, K\}$, let ${\cal S}$ denote the random variable associated with the source domain, with $n$ samples drawn from $p_{\cal S,Y}$, referred to as $\{(x_{i},y_{i})\}_{i=1}^{n}$, and let ${\cal T}$ denote the random variable associated with the target domain, with $m$ samples drawn from $p_{\cal Y,T}$, referred to as $\{(z_{j},u_{j})\}_{j=1}^{m}$. Both the source and target domain are measured in a $D$-dimensional vector space, on the same features. The target labels $u$ are unknown at training time and the goal is to predict them, using only the given unlabeled target samples $\{z_j\}_{j}^{m}$ and the given labeled source samples $\{(x_i,y_i)\}_{i}^{n}$.

\subsection{Target Risk}
The risk minimization framework formalizes \emph{risk}, or the expected loss $\ell$ incurred by classification function $h$, mapping data to classes $h^{\cal S} : {\cal S} \rightarrow {\cal Y}$, with respect to a particular joint labeled data distribution $p_{\cal S,Y} \left(x,y \right)$; $R(h)=\mathbb{E}_{p_{\cal S,Y}} \ \ell \left(h(x),y \right)$. By minimizing empirical risk, i.e. the approximation of the expectation with the sample average over labeled samples $\{(x_i,y_i)\}_i^n$, with respect to classifiers from a space of hypothetical classification functions $H$, one hopes to find the function that generalizes most to novel samples. Additionally, a regularization term that punishes classifier complexity is often incorporated to avoid finding classifiers that are too specific to the given labeled data. For a given data distribution, the choice of loss function, the hypothesis space and amount of regularization largely determine the behavior of the resulting classifier.

\paragraph{}
The empirical risk in the source domain can be computed as follows:
\begin{align}
	\hat{R} \left(h \given x,y \right) = \frac{1}{n} \sum_{i=1}^{n} \ell \left(h \given x_i, y_i \right) \, , \nonumber
\end{align}
with the \emph{source classifier} being the classifier that is found by minimizing this risk:
\begin{align}
	\hat{h}^{\cal S}= \underset{h \in H}{\arg \min} \ \hat{R} \left( h \given x,y \right) \, . \label{hS}
\end{align}
Since the source classifier does not incorporate any target data, it is essentially entirely naive of the target domain. But, if we assume that the domains are related in some way, then it makes sense to apply the source classifier on the target data. To evaluate $\hat{h}^{\cal S}$ in the target domain, the empirical \emph{target risk}, i.e. the risk of the classifier with respect to target samples, is measured:
\begin{align}
	\hat{R} \big( \hat{h}^{\cal S} \given z,u \big) = \frac{1}{m} \sum_{j=1}^{m} \ \ell \big( \hat{h}^{\cal S} \given z_j,u_j \big) \, . \label{hs_RT}
\end{align} 
Training on the source domain and testing on the target domain is our baseline, non-adaptive approach. Although the source classifier does not incorporate information from the target domain nor any knowledge on the relation between the domains, it is often \emph{not the worst} classifier. In cases where approaches rely heavily on assumptions, the adaptive classifiers can deviate from the source classifier in ways that lead to even larger target risks.

\subsection{Contrast}
We are interested in finding a classifier that is never worse than the source classifier in terms of the empirical target risk. We formalize this desire by subtracting the source classifiers target risk in (\ref{hs_RT}) from the target risk of a different classifier $h$:
\begin{align}
	\hat{R} \big( h \given z,u \big) \ - \ \hat{R} \big( \hat{h}^{\cal S} \given z,u \big) \label{contrast} 
\end{align}
If such a contrast is used as a risk minimization objective, i.e. $\underset{h \in H}{\min} \ \hat{R}(h \given z,u) \ - \ \hat{R}(\hat{h}^{\cal S} \given z,u)$,  then the risk of the resulting classifier is bounded above by the risk of the source classifier: the maximal value of the contrast is $0$, which occurs when the same classifier is found, $h = \hat{h}^{\cal S}$. Classifiers that lead to larger target risks are not valid solutions to the minimization problem, which implies that certain parts of the hypothesis space $H$ will never be reached. As such, the contrast implicitly constrains $H$ in a similar way as projection estimators \cite{Krijthe2017ProjectedEF}.

\subsection{Pessimism}
However, (\ref{contrast}) still incorporates the target labels $u$, which are unknown. Taking a conservative approach, we use a worst-case labeling instead, achieved by \emph{maximizing} risk with respect to a hypothetical labeling $q$. For any classifier $h$, the risk with respect to this worst-case labeling will always be larger than the risk with respect to the true target labeling:
\begin{align}
	\hat{R} \left(h \given z, u \right) \leq \underset{q}{\max} \ \hat{R} \left( h \given z, q \right) \, . \label{pess}
\end{align}
Unfortunately, maximizing over a set of discrete labels is a combinatorial problem and is computationally very expensive. To avoid this expense, we represent the hypothetical labeling probabilistically: $q_{kj} := p(y_j =k\given z_{j})$. Such a representation is sometimes also referred to as a \emph{soft} label \cite{kuncheva2001decision}. Additionally, it means that $q_j$ is constrained to be an element of a $K-1$ simplex, $\Delta_{K-1}$. For $m$ samples, there are $m$ simplices: $\Delta_{K-1} \times \Delta_{K-1} \dots = \Delta_{K-1}^{m}$. Note that known labels can also be represented probabilistically, for example $ y_i =1 \ \Leftrightarrow \ p( y_i = 1 \given x_i) = 1$, and are sometimes referred to as \emph{crisp} labels.

\subsection{Target Contrastive Pessimistic Risk}
Joining the contrastive target risk from (\ref{contrast}) with the pessimistic labeling $q$ from (\ref{pess}) forms the following risk function: 
\begin{align}
	\hat{R}^{\text{TCP}} \big( h \given \hat{h}^{\cal S}, z, q \big) = \ \frac{1}{m} \sum_{j=1}^{m} \ell \big(h \given z_j, q_j \big) - \ell \big( \hat{h}^{\cal S} \given z_j, q_j \big) \label{tcp} \, .
\end{align}
We refer to the risk in equation \ref{tcp} as the Target Contrastive Pessimistic risk ({\sc tcp}). Minimizing it with respect to a classifier $h$ and maximizing it with respect to the hypothetical labeling $q$, leads to the new {\sc tcp} target classifier:
\begin{align}
	\hat{h}^{\cal T} = \underset{h \in H}{\arg \min} \underset{q \in \Delta_{K-1}^m}{\max} \hat{R}^{\text{TCP}} \big( h \given \hat{h}^{\cal S},z,q \big) \, . \label{h_T}
\end{align}

\paragraph{}
Note that the {\sc tcp} risk expresses only the performance on the target domain. It is different from the ones used in \cite{liu2014robust} and \cite{wen2014robust}, because those incorporate the classifiers performance on the source domain as well. Our formulation contains no evaluation on the source domain, and focuses solely on the performance gain we can achieve  with respect to the source classifier.

\subsection{Optimization}
If the loss function $\ell$ is restricted to be globally convex and the hypothesis space H is a convex set, then the {\sc tcp} risk with respect to $h$ will be globally convex and there will be a unique optimum with respect to $h$. The {\sc tcp} risk with respect to $q$ is bounded linear due to the simplex, which means that it is possible that the optimum is not unique. Nonetheless, the combination is globally convex-linear and the existence of a saddle point, i.e. an optimum with respect to both $h$ and $q$, for the minimax objective is guaranteed \cite{cherukuri2017saddle}. 

Finding the saddle point can be done through first performing a gradient descent step according to the partial derivative with respect to $h$, followed by a gradient ascent step according to the partial derivative with respect to $q$. However, this last step causes the updated $q$ to leave the simplex. In order to enact the constraint, it is projected back onto the simplex after performing the gradient step. This projection ${\cal P}$ maps a point outside the simplex $a$ to the point on the simplex $b$ that is closest in terms of Euclidean distance: $\mathcal{P}(a) = \underset{b \in \Delta}{\arg \min} \| a - b \|_2$ \cite{chen2011projection,condat2014fast}. Unfortunately, the projection step complicates the computation of the step size, which we replace by a learning rate $\alpha^t$, decreasing over iterations $t$. This results in the overall update: $q^{t+1} \leftarrow {\cal P}(q^{t} + \alpha^{t} \nabla q^{t})$. Lastly, a gradient descent - gradient ascent procedure for globally convex-linear objectives is guaranteed to converge to the saddle point (c.f. proposition 4.4 and corollary 4.5 of \cite{cherukuri2017saddle}).

\section{Least-squares} \label{case_ls}
Discriminative classification models make no assumptions on the data distributions and directly optimize predictions. We incorporate a discriminative model through the least-squares classifier, which is defined by a quadratic loss function $\ell_{\text LS}(h \given x_i,y_i) = (h(x_i)  - y_i)^2$ \cite{friedman2001elements}. For multi-class classification, we employ a one-hot label encoding, also known as a one-vs-all scheme \cite{mohri2012foundations}.

Furthermore, we chose a linear hypothesis space, $h(z) = \sum_{d}^{D} z_d \theta_{kd} + \theta_{k0}$, which we will denote as the inner product $z \theta_k$ between the data row vector, implicitly augmented with a constant $1$, and the classifier parameter vector. $\theta$ is an element of a $(D+1) \times K$-dimensional parameter space $\Theta$ and in the following, we will refer to the classifier optimization step, i.e. minimization over $h \in H$, as a parameter estimation step, i.e. minimization over $\theta \in \Theta$. In summary, the least-squares loss of a sample is:
\begin{align}
\ell_{\text{LS}}(\theta \given z_j, q_j) =& \sum_{k=1}^{K} \left(z_{j}\theta_k - q_{kj} \right)^2 \label{ls} \, .
\end{align}
Plugging (\ref{ls}) into (\ref{tcp}), the {\sc tcp-ls} risk is defined as:
\begin{align}
	 \hat{R}^{\text{TCP}}_{\text{LS}}\big(\theta \given \hat{\theta}^{\cal S}, z, q) =& \ \frac{1}{m} \sum_{j=1}^{m} \ell_{\text{LS}}\big(\theta \given z_j, q_j) - \ell_{\text{LS}}\big(\hat{\theta}^{\cal S} \given z_j, q_j) \nonumber \\
	=& \ \frac{1}{m} \sum_{j=1}^{m} \sum_{k=1}^{K} \big(z_j \theta_k - q_{kj})^2 - \big(z_j \hat{\theta}_k^{\cal S} - q_{kj})^2 \, ,\nonumber
\end{align}
with the resulting estimate:
\begin{align}
	\hat{\theta}_{\text{LS}}^{\cal T} = \ \underset{\theta \in \Theta}{\arg \min} \underset{q \in \Delta_{K-1}^m}{\max} \hat{R}^{\text{TCP}}_{\text{LS}}(\theta \given \hat{\theta}^{\cal S}, z, q) \label{tcp-ls} \, .
\end{align}

\paragraph{}
For fixed $q$, the minimization over $\theta$ has a closed form solution. For each class, the parameter vector is:
\begin{align}
	\frac{\partial}{\partial \theta_k } \ \hat{R}^{\text{TCP}}_{\text{LS}}(\theta \given \hat{\theta}^{\cal S},z, q) =& \ 0 \nonumber \\
	\frac{1}{m} \sum_{j=1}^{m} 2 \ z_{j}^{\top} (z_{j} \theta_k - q_{kj}) \ =& \ 0 \nonumber \\
\theta_k \ =& \ \Big( \sum_{j=1}^{m} z_{j}^{\top} z_{j}\Big)^{-1}\Big( \sum_{j=1}^{m} z_{j}^{\top} q_{kj} \Big) \nonumber \, .
\end{align}

\paragraph{}
Keeping $\theta$ fixed, the gradient with respect to $q_{kj}$ is linear:
\begin{align}
	\frac{\partial}{\partial q_{kj}} \hat{R}^{\text{TCP}}_{\text{LS}}(\theta \given \hat{\theta}^{\cal S}, z, q) =& \ \frac{-2}{m} ( z_j \theta_k - q_{kj}) - \frac{-2}{m} (z_j \hat{\theta}_k^{\cal S} - q_{kj}) \nonumber \\
	=& \ \frac{2}{m} \big( z_j \hat{\theta}^{\cal S}_k - z_j \theta_k \big) \, . \nonumber
\end{align}
Algorithm \ref{algbox_ls} gives pseudo-code for {\sc tcp-ls}.

\begin{algorithm}[ht]
   \caption{{\sc tcp-ls}}
   \label{algbox_ls}
\begin{algorithmic}
   \STATE {\bfseries Input:} source data $x$ (size $n\times D$), labels $y$ (size $n\times K$), target data $z$ (size $m \times D$), learning rate $\alpha$, convergence criterion $\epsilon$.
   \STATE {\bfseries Output:} $\hat{\theta}^{\cal T}_{\text{LS}} = (\theta_1, \dots, \theta_K)$
   \STATE Initialize: 	$q_{kj} \leftarrow 1/K \quad \quad \forall k,j$ 
      \FORALL{classes}
	   \STATE	$\hat{\theta}_k^{\cal S} = \Big(\sum_{i}^{n} x_{i}^{\top} x_{i}\Big)^{-1}\Big( \sum_{i}^{n} x_{i}^{\top} y_{ki} \Big)$ 
   \ENDFOR
   \REPEAT   
   \FORALL{classes}
	     \STATE $\theta_k = \big( \sum_{j}^{m} z_{j}^{\top} z_{j}\big)^{-1}\big( \sum_{j}^{m} z_{j}^{\top} q_{kj}\big)$ 
	      \FORALL{samples}
	      	   \STATE 	$\nabla q_{kj} = 2 \big( z_j \hat{\theta}^{\cal S}_k - z_j \theta_k \big) / m$
	      	   \ENDFOR
  \ENDFOR
   \STATE $q^{t+1} \leftarrow {\cal P} \big(q^{t} - \alpha^{t} \nabla q^{t} \big)$
   \UNTIL{$ \| \ \hat{R}^{\text{TCP}}_{\text{LS}} \big(\theta^{t+1} \given \hat{\theta}^{\cal S},z,q^{t+1} \big) - \hat{R}^{\text{TCP}}_{\text{LS}} \big(\theta^{t} \given \hat{\theta}^{\cal S},z,q^{t} \big) \ \| \leq \epsilon$}
\end{algorithmic}
\end{algorithm}

\section{Discriminant Analysis} \label{case_da}
As a generative classification model, we chose the classical discriminant analysis model ({\sc da}). It fits a Gaussian distribution to each class, $\mathcal{N}(x,y \given \theta_k)$, and classifies new samples $x^*$ according to the largest probability over Gaussians; $h(z) = \underset{k}{\arg \max} \ \mathcal{N}(x^*,k \given \theta_k)$. Again, we will refer to the classifier optimization step as a parameter estimation step. For {\sc da} models, the parameter space $\Theta$ consists of priors, means and covariance matrices for the Gaussian distributions; $\theta_k = (\pi_k, \mu_k, \Sigma_k)$. The model is incorporated in the empirical risk minimization framework by setting the loss function to the negative log-likelihood, $\ell(\theta \given z_j, u_j) = -\log \mathcal{N}(z_j,u_j \given \theta_u)$. The probabilistic labeling $q$ is incorporated by weighing the likelihood over each class' Gaussian distribution: $\ell(\theta \given z_j,q_j) = \sum_{k}^{K} - q_{kj} \log \mathcal{N}(z_j,k \given \theta_k)$.

\subsection{Quadratic Discriminant Analysis}
If one fits one Gaussian distribution per class, the resulting classifier is a quadratic function of the difference in means and covariances, and is hence referred to as quadratic discriminant analysis ({\sc qda}):
\begin{align}
\ell_{\text{QDA}}(\theta & \given z_j, q_j) = \sum_{k=1}^{K} - q_{kj} \log \mathcal{N}(z_{j}, k \given \theta_k) \label{qda} 
\\ =& \sum_{k=1}^{K} - q_{kj} \big[ \log \pi_{k} - \frac{1}{2}\log\big[(2\boldsymbol{\pi})^{D} \det (\Sigma_k) \big] - \frac{1}{2} (z_{j} - \mu_k) \Sigma_k^{-1} (z_{j} - \mu_k)^{\top} \big] \nonumber \, ,
\end{align}
where $\det$ refers to the determinant and $\boldsymbol{\pi}$ refers to the irrational constant.

Plugging the loss from (\ref{qda}) into (\ref{tcp}), the {\sc tcp-qda} risk becomes:
\begin{align}
	\hat{R}_{\text{QDA}}^{\text{TCP}}(\theta \given \hat{\theta}^{\cal S}, z, q) =& \ \frac{1}{m} \sum_{j=1}^{m} \ \ell_{\text{QDA}}(\theta \given z_j, q_j) - \ell_{\text{QDA}}(\hat{\theta}^{\cal S} \given z_j, q_j) \nonumber \\
	=& \ \frac{1}{m} \sum_{j=1}^{m} \sum_{k=1}^{K} - q_{kj} \log \frac{\pi_{k} \ p(z_{j}, k \given \mu_{k}, \Sigma_k)}{\hat{\pi}^{\cal S}_{k} \ p(z_{j}, k \given \hat{\mu}^{\cal S}_{k}, \hat{\Sigma}_k^{\cal S})} \label{tcp-qda} \, \, ,
\end{align}
where the estimate itself is:
\begin{align}
	\hat{\theta}_{\text{QDA}}^{\cal T} = \underset{\theta \in \Theta}{\arg \min} \underset{q \in \Delta^{m}_{K-1}}{\max} \hat{R}_{\text{QDA}}^{\text{TCP}} (\theta \given \hat{\theta}^{\cal S}, z, q) \, . \nonumber
\end{align}

\paragraph{}
Minimization with respect to $\theta$ also has a closed-form solution for discriminant analysis models. For each class, the parameter estimates are:
\begin{align}
	\pi_{k} &= \frac{1}{m} \sum_{j=1}^{m} q_{kj} \, , \nonumber \\
	\mu_{k} &= \big( \sum_{j=1}^{m} q_{kj} \big)^{-1} \sum_{j=1}^{m} q_{kj} z_{j}  \, , \nonumber \\
	\Sigma_k &= \big( \sum_{j=1}^{m} q_{kj} \big)^{-1} \sum_{j=1}^{m} q_{kj}(z_{j} - \mu_{k})^{\top} (z_{j} - \mu_{k}) \, . \nonumber
\end{align}
One of the properties of a discriminant analysis model is that it requires the estimated covariance matrix $\Sigma_k$ to be non-singular. It is possible for the maximizer over $q$ in {\sc tcp-qda} to assign less samples than dimensions to one of the classes, causing the covariance matrix for that class to be singular. To prevent this, we regularize its estimation by first restricting $\Sigma_k$ to minimal eigenvalues of $0$ and then adding a scalar multiple of the identity matrix $\lambda I$. Essentially, the estimated covariance matrix is constrained to a minimum size in each direction. 

\paragraph{}
Keeping $\theta$ fixed, the gradient with respect to $q_{kj}$ is linear:
\begin{align}
	\frac{\partial}{\partial q_{kj}} \hat{R}^{\text{TCP}}_{\text{QDA}} \big( \theta \given \hat{\theta}^{\cal S},z,q) = \ - \frac{1}{m}\log \frac{\pi_{k} \ p(z_{j},k\given \mu_{k}, \Sigma_k)}{\hat{\pi}_{k}^{S} \ p(z_{j},k \given \hat{\mu}_{k}^{S}, \hat{\Sigma}_k^{S})} \, \, . \nonumber
\end{align}
Algorithm \ref{algbox_qda} lists pseudo-code for {\sc tcp-qda}.

\begin{algorithm}[ht]
   \caption{{\sc tcp-qda}}
   \label{algbox_qda}
\begin{algorithmic}
   \STATE {\bfseries Input:} source data $x$ (size $n\times D$), labels $y$ (size $n\times K$), target data $z$ (size $m \times D$), learning rate $\alpha$, convergence criterion $\epsilon$.
   \STATE {\bfseries Output:} $\hat{\theta}^{\cal T}_{\text{QDA}} = (\pi_1, \dots, \pi_K, \mu_1, \dots, \mu_K, \Sigma_1, \dots, \Sigma_K)$
   \STATE Initialize: $q_{kj} \leftarrow 1/K \quad \quad \forall k,j$  
   \FORALL{classes} 
	   \STATE $\hat{\pi}^{\cal S}_{k} = n^{-1} \sum_{i}^{n} y_{ki}$   
	   \STATE $\hat{\mu}^{\cal S}_{k} = \big( \sum_{i}^{n} y_{ki} \big)^{-1} \sum_{i}^{n} y_{ki}x_{i}$   
	   \STATE $\hat{\Sigma}_k^{\cal S} = \big( \sum_{i}^{n} y_{ki} \big)^{-1} \sum_{i}^{n} y_{ki}(x_{i} - \hat{\mu}^{\cal S}_{k})^{\top}(x_{i} - \hat{\mu}^{\cal S}_{k})$   
   \ENDFOR
%   \STATE $\hat{\theta}^{\cal S} = (\hat{\pi}_1^{\cal S}, \dots, \hat{\pi}_K^{\cal S}, \hat{\mu}_1^{\cal S}, \dots, \hat{\mu}_K^{\cal S}, \hat{\Sigma}_1^{\cal S}, \dots, \hat{\Sigma}_K^{\cal S})$
   \REPEAT   
   \FORALL{classes}
	   \STATE $\pi_{k} = m^{-1} \sum_{j}^{m} q_{kj} $   
	   \STATE $\mu_{k} = \big( \sum_{j}^{m} q_{kj} \big)^{-1} \sum_{j}^{m} q_{kj} z_{j}$   
	   \STATE $\Sigma_k = \big( \sum_{j}^{m} q_{kj} \big)^{-1}  \sum_{j}^{m} q_{kj} (z_{j} - \mu_{k})^{\top}(z_{j} - \mu_{k})$   
	   \FORALL{samples}
	      \STATE $\nabla q_{kj} = -\log \big[ \ \pi_{k} \ p(z_{j},k\given \mu_{k}, \Sigma_k) \ / \ \hat{\pi}^{\cal S}_{k} \ p(z_{j}, k \given  \hat{\mu}^{\cal S}_{k}, \hat{\Sigma}_k^{\cal S}) \ \big]$
	      \ENDFOR
   \ENDFOR
   \STATE $q^{t+1} \leftarrow {\cal P}(q^{t} - \alpha^{t} \nabla q^{t})$
   \UNTIL{$ \| \ \hat{R}^{\text{TCP}}_{\text{QDA}} \big(\theta^{t+1} \given \hat{\theta}^{\cal S}, z,q^{t+1} \big) - \hat{R}^{\text{TCP}}_{\text{QDA}} \big(\theta^{t} \given \hat{\theta}^{\cal S},z, q^{t} \big) \ \| \leq \epsilon$}
\end{algorithmic}
\end{algorithm}

\subsection{Linear Discriminant Analysis}
If one constrains the model to share a single covariance matrix for each class, the resulting classifier is a linear function of the difference in means and is hence termed linear discriminant analysis ({\sc lda}). This constraint is enforced through the weighted sum over class covariance matrices $\Sigma = \sum_{k}^{K} \pi_k \Sigma_k$. 

%\paragraph{} 
%The parameter estimates for {\sc lda} are exactly the same
%\begin{align}
%	\hat{\pi}_{k} &= \frac{1}{m} \sum_{j=1}^{m} q_{kj} \, , \nonumber \\
%	\hat{\mu}_{k} &= \big( \sum_{j=1}^{m}q_{kj} \big)^{-1} \sum_{j=1}^{m} q_{kj} z_{j}  \, , \nonumber \\
%	\hat{\Sigma} &= \sum_{k=1}^{K} \hat{\pi}_k \frac{1}{m} \sum_{j=1}^{m} q_{kj}(z_{j} - \hat{\mu}_{k})(z_{j} - \hat{\mu}_{k})^{\top} \, . \nonumber
%\end{align}
%
%\paragraph{} 
%Again, keeping $\hat{\theta}$ fixed, the gradient with respect to $q_{kj}$ is linear:
%\begin{align}
%	\frac{\partial}{\partial q_{kj}} \hat{R}^{\cal T}_{\text{LDA}}(h,\hat{h}^{\cal S}, q) = \ - \frac{1}{m}\log \frac{\hat{\pi}_{k} \ p(z_{j},k\given \hat{\mu}_{k}, \hat{\Sigma})}{\hat{\pi}_{k}^{S} \ p(z_{j},k \given \hat{\mu}_{k}^{S}, \hat{\Sigma}^{S})} \, \, . \nonumber
%\end{align}

\subsection{Performance Guarantee}
The discriminant analysis model has a very surprising property: it obtains a \emph{strictly} smaller risk than the source classifier. To our knowledge, this is the first time that such a performance guarantee can be given in the context of domain adaptation, without using any assumptions on the relation between the two domains. 
\begin{theorem}
For a continuous target distribution, with more unique samples than features for every class, $m_k > D$, the empirical target risk of a discriminant analysis model $\hat{R}_{\text{DA}}$ with {\sc tcp} estimated parameters $\hat{\theta}^{\cal T}$ is strictly smaller than the empirical target risk of a discriminant analysis model with parameters $\hat{\theta}^{\cal S}$ estimated on the source domain:
\begin{align}
	\hat{R}_{\text{DA}} \big( \hat{\theta}^{\cal T} \given z,u \big) \ < \ \hat{R}_{\text{DA}} \big( \hat{\theta}^{\cal S} \given z,u \big) \nonumber
\end{align}
\end{theorem}

The reader is referred to \nameref{proof} for the proof. It follows similar steps as a robust guarantee for discriminant analysis in semi-supervised learning \cite{loog2016contrastive}. It should be noted that the risks of {\sc tcp-lda} and {\sc tcp-qda} are always strictly smaller with respect to the given target samples, but not necessarily strictly smaller with respect to new target samples. Although, when the given target samples are a good representation of the target distribution, one does expect the adapted model to generalize well to new target samples. Additionally, as long as the same amount of regularization $\lambda$ is added to both the source $\hat{\theta}^{\cal S}$ and the {\sc tcp} classifier $\hat{\theta}^{\cal T}$, the guarantee also holds for a regularized model.

\section{Experiments}
Our experiments compare the risks of the {\sc tcp} classifiers with that of the source classifier and the corresponding oracle target classifier, as well as their performance with respect to various state-of-the-art domain adaptive classifiers through their areas under the ROC-curve. In all experiments, all target samples are given, unlabeled, to the adaptive classifiers. They make predictions for those given target samples and their performance is evaluated with respect to those target samples' true labels. Cross-validation for regularization parameters was done by holding out source data, as that is the only data for which labels are available at training time. The range of values we tested was $\begin{bmatrix} 0 \ 10^{-6} \ 10^{-5} \ 10^{-4} \ 10^{-3} \ 10^{-2} \ 10^{-1} \ 10^{0} \ 10^{1} \ 10^{2} \ 10^{3} \end{bmatrix}$.

\subsection{Compared methods}
We implemented transfer component analysis ({\sc tca}) \cite{pan2011domain}, kernel mean matching ({\sc kmm}) \cite{huang2007correcting}, robust covariate shift adjustment ({\sc rcsa}) \cite{wen2014robust} and the robust bias-aware ({\sc rba}) classifier  \cite{liu2014robust} for the comparison (see cited papers for more information). {\sc tca} and {\sc kmm} are chosen because they are popular classifiers with clear assumptions. {\sc rcsa} and {\sc rba} are chosen because they also employ minimax formulations but from different perspectives; {\sc rcsa} as a worst-case and {\sc rba} as a moment-matching importance weighing. Their implementations details are discussed shortly below.

\paragraph{Transfer Component Analysis}
{\sc tca} assumes that there exists a common latent representation for both domains and aims to find this representation by means of a cross-domain nonlinear component analysis \cite{pan2011domain}. In our implementation, we employ a radial basis function kernel with a bandwidth of $1$ and set the trade-off parameter $\mu$ to $1/2$. After mapping the data onto their transfer components, we train a logistic regressor on the mapped source data and apply it to the mapped target data.

\paragraph{Kernel Mean Matching}
{\sc kmm} assumes that the class-posterior distributions are equal in both domains and that the support of the target distribution is contained within the source distribution \cite{huang2007correcting,gretton2009covariate}. When the first assumption fails, {\sc kmm} will have deviated from the source classifier in a manner that will not lead to better results on the target domain. When the second assumptions fails, the variance of the importance-weights increases to the point where a few samples receive large weights and all other samples receive very small weights, reducing the effective training sample size and leading to pathological classifiers. We use a radial basis function kernel with a bandwidth of $1$, kernel regularization of $0.001$ to favor estimates with lower variation over weights and upper bound the weights by $10 \ 000$. After estimating importance weights, we train a weighed least-squares classifier on the source samples.
 
 \paragraph{Robust Covariate Shift Adjustment} 
{\sc rcsa} also assumes equal class-posterior distributions and containment of the support of the target distribution within the source distribution, but additionally incorporates a worst-case labeling \cite{wen2014robust}. To be precise, it maximizes risk with respect to the importance weights. We used the author's publicly available code with 5-fold cross-validation for its hyperparameters. Interestingly, the authors also discuss a relation between covariate shift and model misspecification, as described by \cite{white1981consequences}. They argue for a two-step (estimate weights - train classifier) approach in a game-theoretical form \cite{grunwald2004game,sugiyama2005model,wen2014robust}, which is done by all importance-weighted classifiers in this paper.
 
 \paragraph{Robust Bias-Aware}
 {\sc rba} assumes that the moments of the feature statistics are approximately equal up to a particular order \cite{liu2014robust}. In their formulation, the adversary plays a classifier whose class-posterior probabilities are used as a labeling of the target samples, but who is also constrained to match the moments with the source domain's statistics. The player then proposes an importance-weighted classifier that aims to perform well on both domains. Note that the constraints on the adversary are, among others, necessary to avoid the players switching strategies constantly. We implement {\sc rba} using first-order feature statistics for the moment-matching constraints, which was also done by the authors in their paper. Furthermore, we use a ratio of normal distributions for the weights and bound them above by $1000$.

\subsection{Experiments in a sample selection bias setting} \label{exp_ssb}
Sample selection bias settings occur when data is collected locally from a larger population. For regression problems, these settings are usually created through a parametric sampling of the feature space \cite{shimodaira2000improving,huang2007correcting}. We created something similar but for classification problems: samples are concentrated around a certain subset of the feature space, but with equal class priors as the whole set. For each class:
\begin{enumerate}
\item Find the sample closest to the origin; $x_0$.
\item Compute distance $d(x_0,x_k)$ to all other samples of the same class.
\item Draw without replacement  $\pi_k n^{\cal S}$ samples proportional to $\exp(-d(x_0,x_k))$.
\end{enumerate}
where $n^{\cal S}$ denotes the total number of samples to draw and $\pi_k$ refers to the class-prior distributions of the whole set. Note that drawing $\pi_k n$ samples from each class leads to approximately the same class prior distributions in the source domain as the target domain. We chose the squared Mahalanobis distance: $d(x_0,x_k) := (x_0 - x_k) \Sigma^{-1} (x_0 - x_k)^{\top}$, with the covariance matrix estimated on all data, since that takes scale differences between features into account. Figure \ref{ssb_nn} presents an example, showing the first two principal components of the pima diabetes dataset. Red/blue squares denote the selected source samples, black circles denote all samples and the green stars denote the seed points ($x_0$ for each class).
\begin{figure}[ht]
\centering
\includegraphics[width=.9\textwidth]{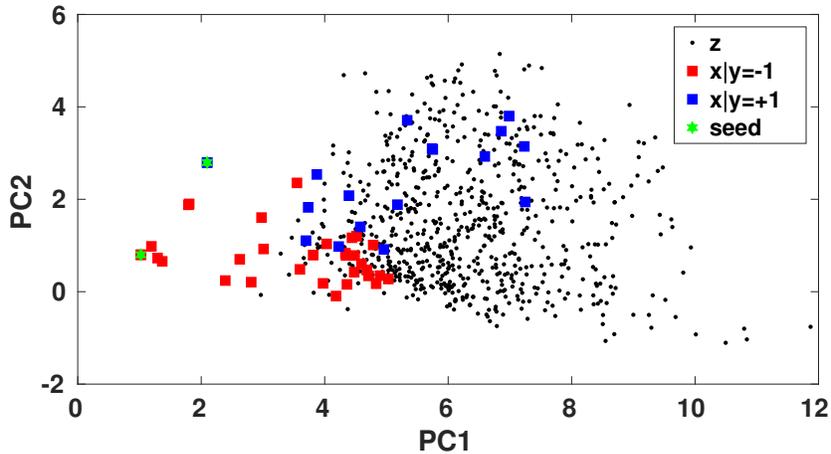}
\caption{Example of a biased sampling. Shown are the first two principal components of the pima diabetes dataset, with all target samples in black, the selected source samples in red/blue and the samples closest to $0$ of each class in green (seeds).}
\label{ssb_nn}
\end{figure}

\subsubsection{Data}
We collected the following datasets from the UCI machine learning repository: cylinder bands printing (bands), car evaluation (car), credit approval (credit), ionosphere (iono), mammographic masses (mamm), pima diabetes (pima) and tic-tac-toe endgame (t3). Table \ref{ssb_char} lists their characteristics. All missing values have been imputed to $0$. For each dataset, we draw $n^{\cal S}=50$ samples as the source domain while treating all samples as the target domain.
\begin{table}[ht]
\caption{Sample selection bias datasets characteristics.}
\label{ssb_char}
\setlength{\tabcolsep}{2pt}
\centering
\begin{tabular}{l}
	\\
\midrule
bands 	\\
car 		\\
credit	\\
iono		\\
mamm	\\
pima	\\
t3 		\\
\end{tabular} \ 
\begin{tabular}{ c c c c }
 \#Samples & \#Features & \#Missing & Class (-1$\vert$+1)  \\
\midrule
 539		& 39		& 569	& 312 $\vert$ 227 		\\
 1728	& 6 	 	& 0	  	& 1210 $\vert$ 518	\\
 690		& 15		& 67		& 307 $\vert$ 383		\\
 351		& 34		& 0		& 126 $\vert$ 225		\\
 961		& 5		& 162	& 516 $\vert$ 445		\\
 768		& 8		& 0		& 500 $\vert$ 268		\\
 958		& 9 	 	& 0	     	& 332 $\vert$ 626	
\end{tabular}
\end{table}

\subsubsection{Results}
The risks (average negative log-likelihoods for the discriminant analysis models and mean squared errors for the least-squares classifiers) in Table \ref{ssb_anll} belong to the source classifiers, the {\sc tcp} classifiers and the oracle target classifiers. The oracles represent the best possible result, as they comprise the risk of a classifier trained on all target samples with their true labels. The results show varying degrees of improvement for the {\sc tcp} classifiers. {\sc tcp-lda} approaches {\sc t-lda} more closely than the other two versions, with {\sc tcp-ls} being the most conservative one. For the ionosphere and tic-tac-toe datasets, the improvement is quite dramatic, indicating that the source classifier is a poor model for the target domain. Note also that some overfitting might be occurring as {\sc tcp-qda} does not always have a lower risk than {\sc tcp-lda}, even though {\sc t-qda} does always have a lower risk than {\sc t-lda}. 
\begin{table}[ht]
\caption{Risks (average negative log-likelihoods and mean squared errors) of the naive source classifiers ({\sc s-lda}, {\sc s-qda}, {\sc s-ls}), the tcp classifiers ({\sc tcp-lda}, {\sc tcp-qda}, {\sc tcp-ls}) and the oracle target classifiers ({\sc t-lda}, {\sc t-qda}, {\sc t-ls}) on the sample selection bias datasets.}
\label{ssb_anll}
\setlength{\tabcolsep}{1pt}
\centering
\begin{tabular}{l}
	\\
\midrule
bands 	\\
car 		\\
credit	\\
iono		\\
mamm	\\
pima	\\
t3 		\\
\end{tabular} \ 
\begin{tabular}{ c c c }
 {\sc s-lda} & {\sc tcp-lda} & {\sc t-lda}\\
\midrule
-216.3 	& -218.4		& -218.8	 	\\
 17.16	& 2.850		& 2.148		\\
-80.04	& -83.64 	 	& -83.65	 	\\
 199.5 	& -8.480		&  -8.782	 	 \\
 8.133	& -10.40		& -11.22	 	\\
-15.92 	& -23.44	 	& -24.15	 	\\
 18.77	& 6.136		& 4.734
\end{tabular} \hfill
\begin{tabular}{c c c }
 {\sc s-qda} & {\sc tcp-qda} & {\sc t-qda}\\
\midrule
-215.3 	& -217.8		& -219.1	 	\\
57.39	&  18.77		&  2.049		\\
 -78.99	& -83.73 	 	& -84.61	 	\\
 26.30	& -9.325		& -18.78	 	 \\
 31.66	& -10.08		& -11.68	 	\\
-7.486 	& -23.09	 	& -24.30	 	\\
 117.3	& 39.13 		& 4.611
\end{tabular} \hfill
\begin{tabular}{c c c }
 {\sc s-ls} & {\sc tcp-ls} & {\sc t-ls}\\
\midrule
1.170 	& 1.109		& 0.827	 	\\
1.968	& 1.205		& 0.672 		\\
2.430	& 0.973 	 	& 0.757	 	\\
17.06 	& 0.815		& 0.350	 	 \\
 0.818	& 0.668		& 0.580	 	\\
1.083 	& 1.012	 	& 0.633	 	\\
 1.401	& 1.401		& 0.849 	 	
\end{tabular}
\end{table}

%Additionally, we constructed a learning curve for the pima diabetes dataset (see Figure \ref{ssb_lc}). On the y-axis we plot the risk, i.e. the average negative log-likelihood, on the whole dataset. The black line indicates the oracle performance, the risk of {\sc t-qda} trained on the fully labeled target set, and the blue line indicates the risk of {\sc s-qda} trained with $n^{\cal S}$ samples on the target set. The red line indicates {\sc tcp-qda} which receives $n^{\cal S}$ source samples and an additional $n^{\cal S}$ unlabeled target samples. It is benefiting from the information in the unlabeled samples, as it converges much faster to the oracle solution.
%\begin{figure}[ht]
%\centering
%\includegraphics[width=.9\textwidth]{pima_lc_tas_risk.eps}
%\caption{Learning curve for the pima diabetes dataset. Average negative log-likelihoods of the source classifier ({\sc s-qda}), the tcp classifier ({\sc tcp-qda}) and the oracle target classifier ({\sc t-qda}) as a function of the number of training samples.}
%\label{ssb_lc}
%\end{figure}

Table \ref{ssb_auc} compares the performances of the adaptive classifiers on all datasets through their area under the ROC-curves (AUC). Although there is quite a variety between datasets, the variation between classifiers within a dataset is relatively small; all approaches perform similarly well. However, with our selection bias procedure, the moments of the target statistics do not match the source statistics (e.g. the target's variance is by construction always larger) which affect {\sc rba}'s performance negatively. Interestingly, the {\sc tcp} discriminant analysis models are quite competitive in cases where their improvement over the source classifier was larger. Unfortunately, like {\sc rba}, the more conservative {\sc tcp-ls} never outperforms all other methods simultaneously on any of the datasets. Still, in the average it reaches competitive performance overall. In summary, the {\sc tcp} classifiers perform on par with the other adaptive classifiers.
\begin{table}[ht]
\caption{Sample selection bias datasets. Areas under the ROC-curves for a range of domain adaptive classifiers.}
\label{ssb_auc}
\setlength{\tabcolsep}{6pt}
\centering
\begin{tabular}{l}
	\\
\midrule
bands 	\\
car 		\\
credit	\\
iono		\\
mamm	\\
pima	\\
t3 		\\
\midrule
mean %\\
%sem 
\end{tabular} \ 
\begin{tabular}{c c c c c c c c}
 {\sc tca} &{\sc kmm} & {\sc rcsa} & {\sc rba} &  {\sc tcp-ls} & {\sc tcp-lda} & {\sc tcp-qda} \\
\midrule
 .578		& {\bf .620}	& .562 		& .504 		& .588		&  .548		& .589		\\
 .736		& {\bf .776}	& .742		& .684		& .734		&  .758		& .699		\\
 {\bf .716}	& .694		& .655 		& .702		& .662		&  .646		& .663		\\
 .741		& .817		& .835 		& .687		& .731		&  {\bf .894}	& .826		\\
 .656		& .804		& .749		& .762		& .836		& .824		& {\bf .847}	\\
 .691		& .630		& {\bf .760} 	& .271		& .692		& .684		& .637		\\
{\bf .608}		& .532		& .439 		& .446		& .520		& .529 		& .606 		\\
\midrule
 .675		& .696		& .677 		& .579		& .680		& {\bf .698}	& .695		%\\
%{\bf .024}		& .041		& .052 		& .067		& .039		& .052 		& .039 		
\end{tabular}
\end{table}

\subsection{Experiments in a domain adaptation setting} \label{exp_da}
We performed a set of experiments on a dataset that is naturally split into multiple domains: predicting heart disease in patients from hospitals in 4 different locations. It is a much more realistic setting because problem variables such as prior shift, class imbalance and proportion of imputed features are not controlled. As such, it is a harder problem than the sample selection bias setting. In this setting, the target domains often only have limited overlap with the source domain and can be very dissimilar. As the results will show, many of the assumptions that the state-of-the-art domain adaptive classifiers rely upon, do not hold and their performance degrades drastically.

\subsubsection{Data}
The hospitals are the Hungarian Institute of Cardiology in Budapest (data collected by Andras Janosi), the University Hospital Zurich (collected by William Steinbrunn), the University Hospital Basel (courtesy of Matthias Pfisterer), the Veterans Affairs Medical Center in Long Beach, California, USA, and the Cleveland Clinic Foundation in Cleveland, Ohio, USA (both courtesy of Robert Detrano), which will be referred to as Hungary, Switzerland, California and Ohio hereafter. The data from these hospitals can be considered domains as the patients are all measured on the same biometrics but show different distributions. For example, patients in Hungary are on average younger than patients from Switzerland (48 versus 55 years). Each patient is described by 13 features: age, sex, chest pain type, resting blood pressure, cholesterol level, high fasting blood sugar, resting electrocardiography, maximum heart rate, exercise-induced angina, exercise-induced ST depression, slope of peak exercise ST, number of major vessels in fluoroscopy, and normal/defective/reversible heart rate. Table \ref{hdis_props} describes the number of samples ($n$, $m$), total number of missing measurements that have been imputed ($mis_{\cal S}$, $mis_{\cal T}$) the class balance ($c_{\cal S}$, $c_{\cal T}$) and the empirical Maximum Mean Discrepancy for all pairwise combinations of designating one domain as the source and another as the target. First of all, the sample size imbalance is not really a problem, as the largest difference occurs in the Ohio - Switzerland combination with 303 and 123 samples respectively. However, the fact that the classes are severely imbalanced in different proportions, for example going from 54\% : 46\% to 7\% : 93\% in Ohio - Switzerland, creates a very difficult setting. A shift in the prior distributions can be disastrous for some classifiers, such as {\sc rba} which relies on matching the source and target feature statistics. Furthermore, a sudden increase in the amount of missing values (unmeasured patient biometrics), such as in Ohio - California, means that a classifier relying on a certain feature for discrimination degrades when this feature is missing in the target domain. Additionally, the empirical Maximum Mean Discrepancy measures how far apart two sets of samples are: $\hat{\text{MMD}} = \| n^{-1} \sum_{i}^{n} \phi(x_i) - m^{-1} \sum_{j}^{m} \phi(z_j) \|^{2} = n^{-2} \sum_{i,i'}^n K(x_i,x_{i'} ) - 2(nm)^{-1} \sum_{i,j} K(x_i,z_j ) + m^{-2} \sum_{j,j'}^{m} K(z_j,z_{j'} )$ \cite{huang2007correcting}. For its kernel, we used a radial-basis function with a bandwidth of 1. An MMD of $0$ means that the two sets are identical, while larger values indicate larger discrepancies between the two sets. The combinations Ohio - Switzerland and Switzerland - Hungary have an MMD that is two orders of magnitude larger than other combinations. Overall, looking at all three sets of descriptive statistics, the combinations Ohio - Switzerland and Switzerland - Hungary should pose the most difficulty for the adaptive classifiers.

Lastly, to further illustrate how the domains differ, we plotted histograms of the age and resting blood pressure of all patients, split by domain (see Figure \ref{domains}). Not only are they different on average, they tend to differ in variance and skewness as well.
\begin{table}[ht]  
\caption{Heart disease dataset properties, for all pairwise domain combinations (O='Ohio', C='California', H='Hungary' and S='Switzerland'). ${\cal S}$ denotes the source and ${\cal T}$ the target domain, $n$ the amount of source and $m$ the amount of target samples, $c_{\cal S}$ the class balance (-1,+1) in the source domain and $c_{\cal T}$ the class balance in the target domain. MMD denotes the empirical Maximum Mean Discrepancy between the source and target data.}    
\label{hdis_props}
\centering                               
\begin{tabular}{l l }
${\cal S}$ & ${\cal T}$ \\
\midrule
O & H \\
O & S \\
O & C \\
H & S \\
H & C \\
S & C \\
H & O \\
S & O \\
C & O \\
S & H \\
C & H  \\
C & S 
\end{tabular} \quad                
\begin{tabular}{ c c c c c c c}
 n & m & $mis_{\cal S}$ & $mis_{\cal T}$ & $c_{\cal S}$ & $c_{\cal T}$ & MMD\\
\midrule
 303		& 294	& 6		& 782	& 164:139	& 188:106	& 0.0012	\\
 303 		& 123 	& 6 		& 273	& 164:139	& 8:115		& 0.1602	\\
 303 		& 200 	& 6 		& 698	& 164:139 	& 51:149 		& 0.0227 \\
 294		& 123	& 782 	& 273	& 188:106	& 8:115		& 0.1384	\\
 294		& 200	& 782	& 698	& 188:106 	& 51:149		& 0.0151	\\
 123		& 200	& 273	& 698	& 8:115		& 51:149		& 0.0804	\\
 294		& 303	& 782	& 6		& 188:106	& 164:139	& 0.0012	\\
 123		& 303	& 273	& 6		& 8:115		& 164:139	& 0.1602	\\
 200		& 303	& 698	& 6		& 51:149		& 164:139	& 0.0227	\\
 123		& 294	& 273	& 782	& 8:115		& 188:106	& 0.1384	\\	
 200 		& 294	& 698	& 782	& 51:149		& 188:106	& 0.0151	\\
 200		& 123	& 698	& 273	& 51:149		& 8:115		& 0.0804	
\end{tabular} 
\end{table}

%\begin{figure}%
%\centering
%\begin{minipage}{0.48\textwidth}%
%	\includegraphics[width=1\textwidth]{histogram_domains_age.eps} \\
%	\includegraphics[width=1\textwidth]{histogram_domains_bp.eps} 
%\end{minipage}%
%\hfill
%\parbox{0.48\textwidth}{
%\includegraphics[width=.5\textwidth]{scatter_domains.eps}
%}%
%\caption{(Left top) Histogram of the age of patients in each domain, (left bottom) histogram of the resting blood pressure of patients in each domain. (Right) Scatterplot of the first two principal components of the heart disease dataset split by domains.}
%\label{domains}%
%\end{figure}

\begin{figure}%
\centering
\includegraphics[width=.48\textwidth]{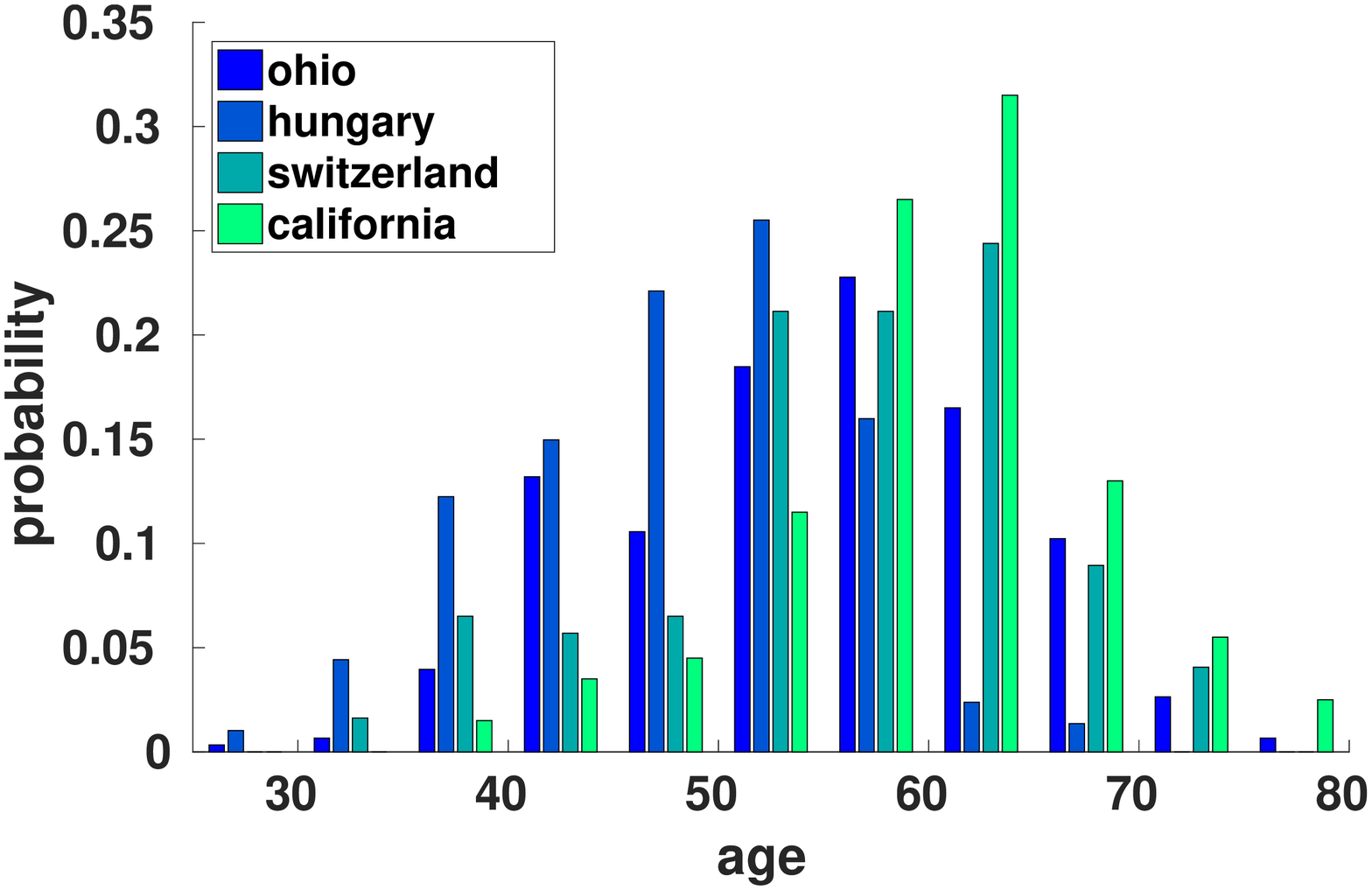} \hfill
\includegraphics[width=.48\textwidth]{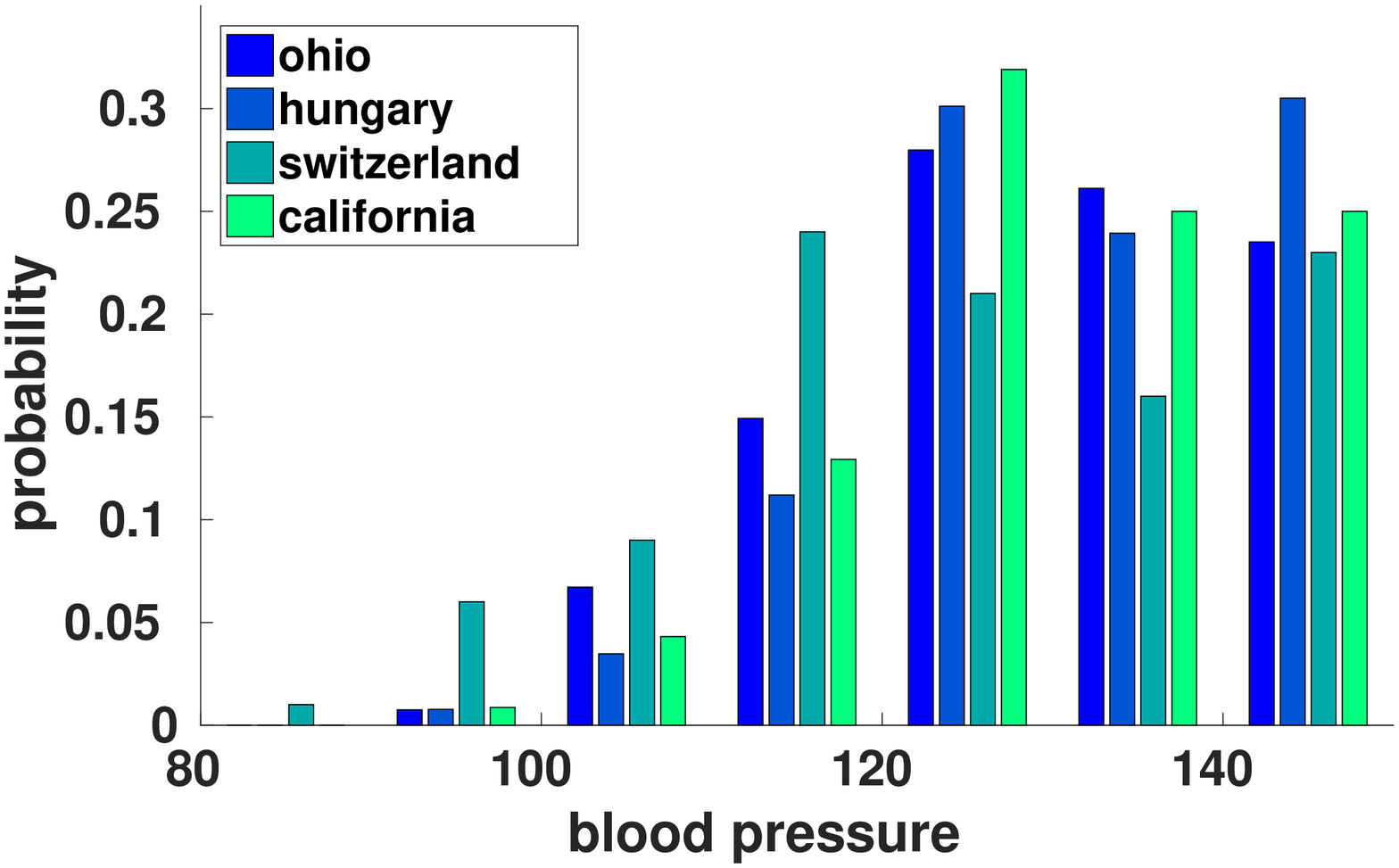} 
\caption{(Left top) Histogram of the age of patients in each domain, (right) histogram of the resting blood pressure of patients in each domain.}
\label{domains}%
\end{figure}

\subsubsection{Results}
Table \ref{hdis_anll_mse} lists the target risks (average negative log-likelihoods for the discriminant analysis models and mean squared errors for the least-squares classifiers) with the given target samples' true labels for all source, {\sc tcp} and oracle target classifiers. Note that the {\sc tcp} risks range between the source and the oracle target risk. For some combinations {\sc tcp} is extremely conservative, e.g. Switzerland - Ohio, Switzerland - Hungary for the least-squares case, and for others, it is much more liberal, e.g. Hungary - Switzerland, Hungary - Ohio, Hungary - California for the discriminant analysis models. In general, the discriminative model ({\sc tcp-ls}) deviates much less and is much more conservative than the generative models ({\sc tcp-lda} and {\sc tcp-qda}). Note that the order of magnitude of the improvement with {\sc tcp-da} in the Hungary - Switzerland, Hungary - Ohio and Hungary - California settings is due to the fact that the two domains lie far apart; the target samples lie very far in the tails of the source models' Gaussian distribution and evaluate to very small likelihoods, which become very large negative log-likelihoods.

\begin{table}[ht]  
\caption{Heart disease dataset. Target risks (average negative log-likelihoods (left, middle) and mean squared errors (right)) for all pairwise combinations of domains (O='Ohio', C='California', H='Hungary' and S='Switzerland'; smaller values are better).}    
\label{hdis_anll_mse}
\setlength{\tabcolsep}{1pt}
\centering                                               
\begin{tabular}{l l }
${\cal S}$ & ${\cal T}$ \\
\midrule
O & H \\
O & S \\
O & C \\
H & S \\
H & C \\
S & C \\
H & O \\
S & O \\
C & O \\
S & H \\
C & H  \\
C & S
\end{tabular} \hfill    
\begin{tabular}{ c c c }                                       
 {\sc s-lda} & {\sc tcp-lda} & {\sc t-lda}\\
\midrule
 -53.55 	& -57.18 	& -57.35 	\\
 -8.293 	& -16.76 	& -17.54 	\\
 -37.84	& -53.88 	& -54.69 	\\
 -12.50 	& -16.08 	& -17.54 	\\
 -41.70 	& -53.91 	& -54.69 	\\
 494.9	& -54.49 	& -54.69 	\\
 -48.91 	& -55.08 	& -55.23 	\\
 709.9 	& -54.07 	& -55.23 	\\
 -49.21 	& -55.00 	& -55.23 	\\
 649.9 	& -56.09 	& -57.35 	\\
 -53.05 	& -57.19 	& -57.35 	\\
 -15.45	& -17.43 	& -17.54
\end{tabular} \hfill    
\begin{tabular}{ c c c }                                       
 {\sc s-qda} & {\sc tcp-qda} & {\sc t-qda}\\
\midrule
 -53.55 	& -57.20 	& -57.62 \\
 -8.293 	& -16.76 	& -17.54 \\
 -37.83 	& -53.73	& -54.89 \\
 -12.80 	& -16.44 	& -17.54 \\
 -40.08 	& -54.45 	& -54.89 \\
 498.9 	& -54.44 	& -54.89 \\
 -49.20 	& -54.84 	& -55.53 \\
 709.9 	& -54.10 	& -55.53 \\
-49.17 	& -55.05	& -55.53 \\
 650.3 	& -56.19 	& -57.62 \\
-53.15	& -57.17 	& -57.62 \\
-15.47 	& -17.44 	& -17.54`
\end{tabular} \hfill  
\begin{tabular}{  c c c }                                       
 {\sc s-ls} & {\sc tcp-ls} & {\sc t-ls}\\
\midrule
 0.580 & 0.579 & 0.444 \\
 1.449 & 1.449 & 0.213 \\
 1.441 & 1.441 & 0.613 \\
 1.068 & 1.068 & 0.213 \\
 1.120 & 1.104 & 0.613 \\
 0.904 & 0.904 & 0.671 \\
 0.642 & 0.638 & 0.463 \\
 1.700 & 1.700 & 0.696 \\
 1.833 & 1.833 & 0.470 \\
 2.102 & 2.102 & 0.740 \\
 0.582 & 0.582 & 0.444 \\
 0.415 & 0.415 & 0.236
\end{tabular}
\end{table} 

Looking at the areas under the ROC-curves in Figure \ref{hdis_auc}, one observes a different pattern in the classifier performances. {\sc tca}, {\sc kmm}, {\sc rcsa} and {\sc rba} perform much worse, often below chance level. It can be seen that, in some cases, the assumption of equal class-posterior distributions still holds approximately, as {\sc kmm} and {\sc rcsa} sometimes perform quite well, e.g. in Hungary - Ohio. {\sc tca}'s performance varies around chance level, indicating that it is difficult to recover a common latent representation in these settings. That makes sense, as the domains lie further apart this time. {\sc rba}'s performance drops most in cases where the differences in priors and proportions of missing values are largest, e.g. Hungary - California, which also makes sense as it is expecting similar feature statistics in both domains. {\sc tcp-ls} performs very well in almost all cases; the conservative strategy is paying off. {\sc tcp-lda} is also performing very well, even outperforming {\sc tcp-qda} in all cases. The added flexibility of a covariance matrix per class is not beneficial because it is much more difficult to fit correctly. Note that the domain combinations are asymmetrical; for example, {\sc rcsa}'s performance is quite strong when Switzerland is the source domain and Ohio the target domain, but it's performance is much weaker when Ohio is the source domain and Switzerland the target domain. In some combinations, assumptions on how two domains are related to each other might be valid that are not valid in their reverse combinations. Overall, in this more general domain adaptation setting, our more conservative approach works best, as shown by the mean performances.
\begin{table}[ht]   
\caption{Heart disease dataset. Area under the ROC-curve for all pairwise combinations of domains (O='Ohio', C='California', H='Hungary' and S='Switzerland'; larger values are better.}    
\label{hdis_auc}
\centering                 
\setlength{\tabcolsep}{5pt}                              
\begin{tabular}{l l }
${\cal S}$ & ${\cal T}$ \\
\midrule
O & H \\
O & S \\
O & C \\
H & S \\
H & C \\
S & C \\
H & O \\
S & O \\
C & O \\
S & H \\
C & H  \\
C & S \\
\midrule
\multicolumn{2}{l}{mean} %\\
%\multicolumn{2}{l}{sem} 
\end{tabular} \quad
\begin{tabular}{c c c c c c c c}   
 {\sc tca} & {\sc kmm} & {\sc rcsa} & {\sc rba}  & {\sc tcp-ls} &  {\sc tcp-lda} & {\sc tcp-qda} \\
\midrule
 .699		& .710		& .372		& .481		& .881		& {\bf .882}	& .817	 	\\
 .590		& .551		& .634		& .670		& {\bf .714}	& .671 		& .671	 	\\
 .496		& .476		& .560		& .450		& {\bf .671}	& .668		& .476 		\\
 .455		& .501		& .646		& .602		& {\bf .668}	& .665 		& .666	 	\\
 .528		& .533		& .585		& .434		& {\bf .727}	& .709 		& .662 		\\
 .475		& .573		& .464		& .603		& {\bf .605}	& .546		& .480 		\\
 .616		& .742		& .751		& .510		& .864		& {\bf .876}	& .863	 	\\	
 .582		& .353		& .750		& .449		& {\bf .753}	& .589 		& .426	 	\\
 .484		& .337		& .551		& .557		& .671		& {\bf .831}	& .828		\\
 .407		& .370		& .629		& .484		& .697		& {\bf 724}	& .604	 	\\
 .472		& .427		& .538		& .616		& .805		& {\bf .878}	& .824		\\
 .511		& .593		& .462		& .474		& {\bf .709}	& .503		& .535		\\
 \midrule
 .526		& .514		& .578		& .528		& {\bf .730}	& .712		& .654		%\\
% {\bf .024}	& .038		& .033		& {\bf .024}	& {\bf .024}	& .038		& .044		
\end{tabular}
\end{table}

\subsubsection{Visualization of the worst-case labeling}
The adversary in {\sc tcp}'s minimax formulation maximizes the objective with respect to the probability $q_{kj}$ that a sample $j$ belongs to class $k$. However, note that the worst-case labeling corresponds to the labeling that maximizes the contrast: it looks for the labeling for which the difference between the source parameters and the current parameters is largest. It would be interesting to visualize this labeling at the saddle point. Figure \ref{wc_lab} shows the first two principal components of Hungary, with the probabilities of belonging to class $1$, i.e. $q_{j k=1}$. The top left figure shows the true labeling, the top right the probabilities for {\sc tcp-ls}, the bottom left for {\sc tcp-lda} and the bottom right for {\sc tcp-qda}. In all three {\sc tcp} cases the labeling is quite smooth and does not vary too much between neighboring points. One would expect a rough labeling, but note that labellings that are bad for the source classifier will most likely also be bad for the {\sc tcp} classifier, and the resulting contrast will be small instead of maximal. The probabilities for {\sc tcp-ls} lie closer to $0$ and $1$ than for {\sc tcp-lda} and {\sc tcp-qda}.

\begin{figure}[ht]
\centering
\includegraphics[width=.45\textwidth]{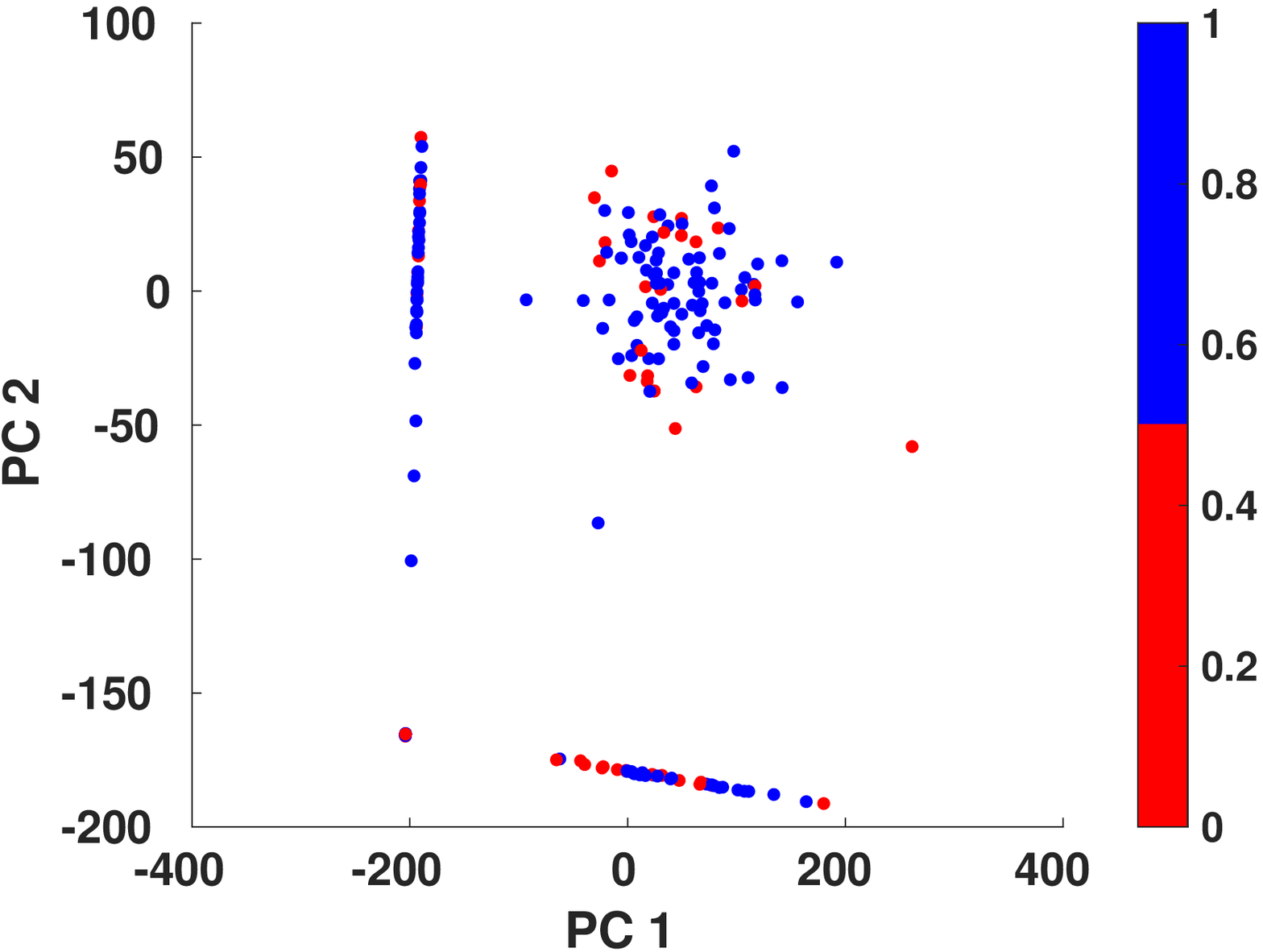} \hfill
\includegraphics[width=.45\textwidth]{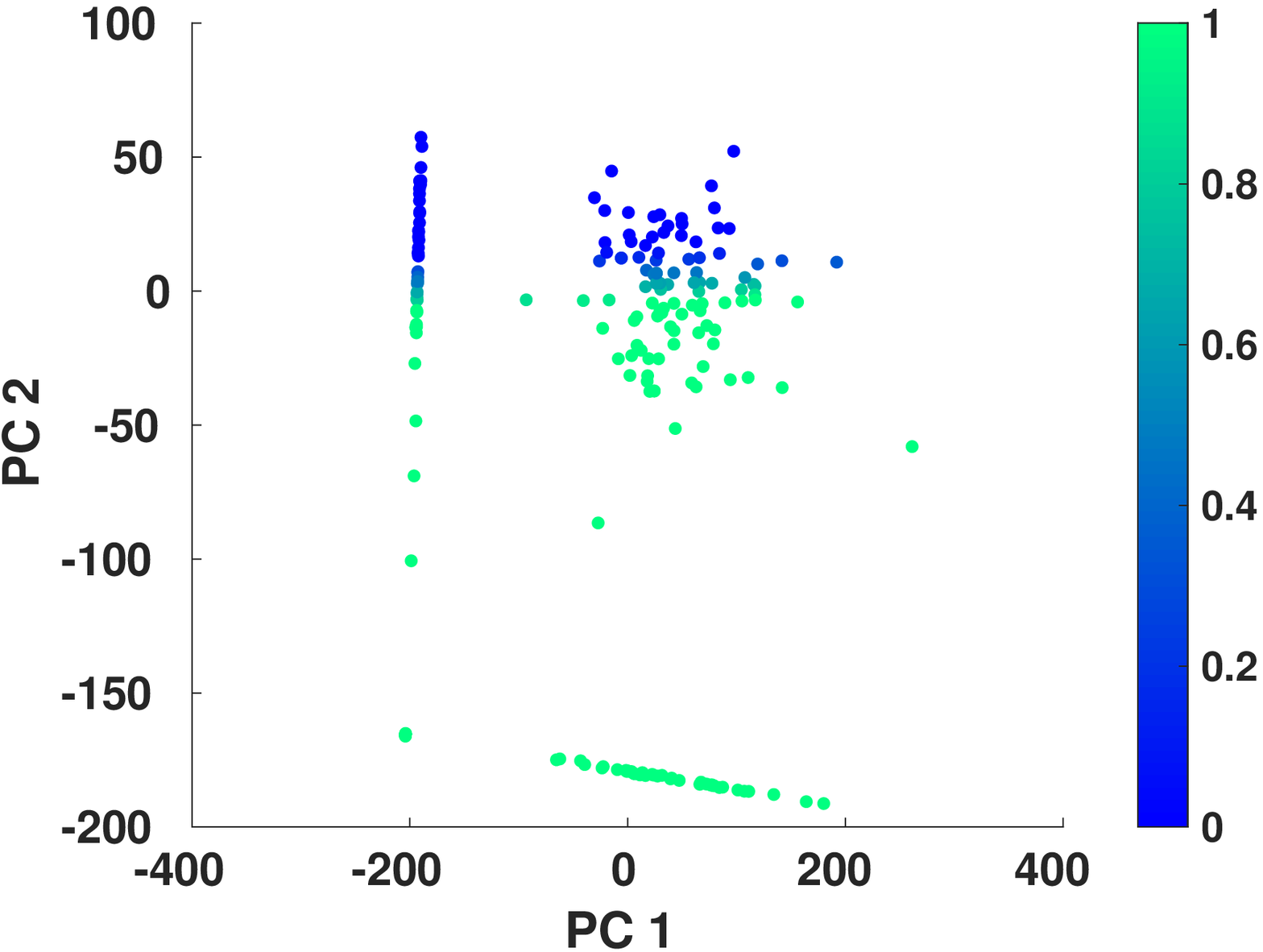} \\
\includegraphics[width=.45\textwidth]{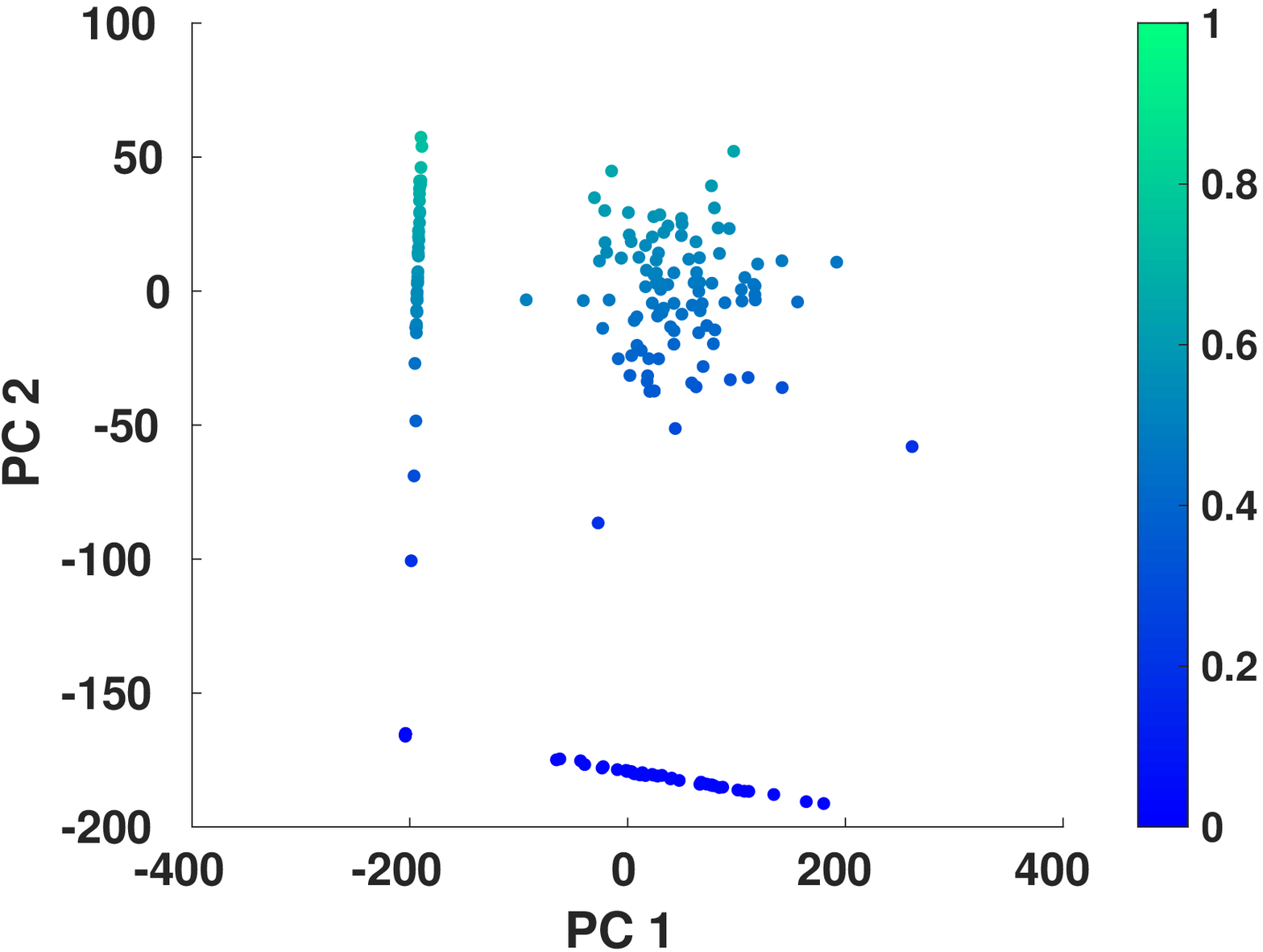} \hfill
\includegraphics[width=.45\textwidth]{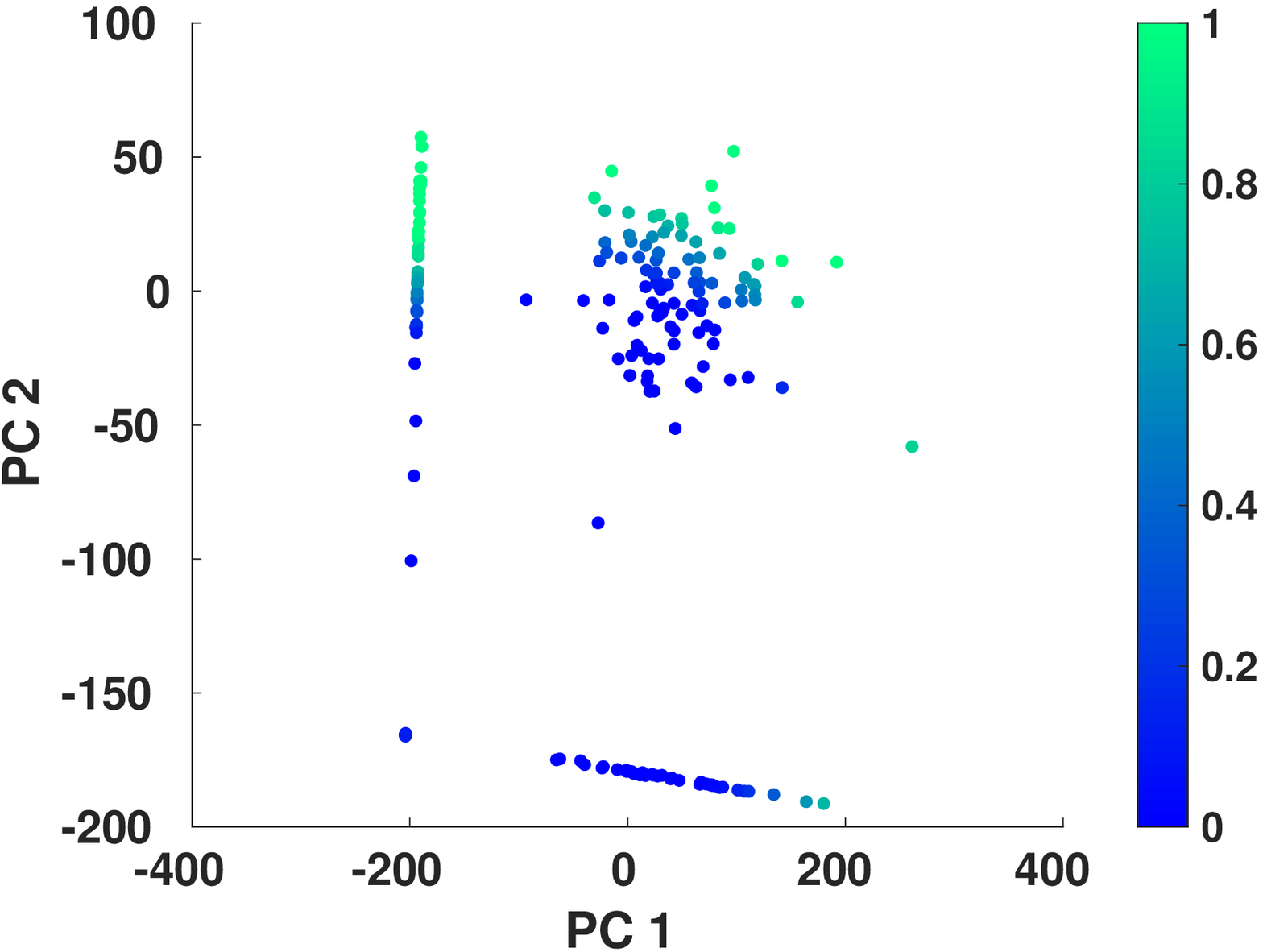}
\caption{Scatter plots of the first two principal components of Hungary in the heart disease dataset. (Top left) True labeling, (top right) $q_{k=1}$ for {\sc tcp-ls}, (bottom left) $q_{k=1}$ for {\sc tcp-lda}, (bottom right) $q_{k=1}$ for {\sc tcp-qda}.}
\label{wc_lab}
\end{figure}

\section{Discussion} \label{disc}
Although the {\sc tcp} classifiers are never worse than the source classifier by construction, they will not automatically lead to improvements in the error rate. This is due to the difference between optimizing a surrogate loss and evaluating the $0$/$1$-loss \cite{bartlett2003prediction, bartlett2006convexity, loog2015semi}. There is no one-to-one mapping between the minimizer of the surrogate loss and the minimizer of the $0$/$1$-loss. 

One peculiar advantage of our {\sc tcp} model is that we do not explicitly require source samples at training time. They are not incorporated in the risk formulation, which means that they do not have to be retained in memory. It is sufficient to receive the parameters of a trained classifier that can serve as a baseline. Our approach is therefore more efficient than for example importance-weighing techniques which require source samples for importance-weight estimation and subsequent training. Additionally, it would be interesting to construct a contrast with multiple source domains. The union of source classifiers might serve as a very good starting point for the {\sc tcp} model.

For each adaptive classifier in this paper, regularization parameters are estimated through cross-validation on held-out source samples. However, this procedure is known to be biased as it does not account for domain dissimilarity \cite{sugiyama2007covariate,kouw2016regularization}. What is optimal with respect to held-out source samples, need not be optimal with respect to target samples. Performance of many adaptive models might be improved with appropriate adaptive validation techniques.

\section{Conclusion}
We have designed a risk minimization formulation for a domain-adaptive classifier whose performance, in terms of risk, is always at least as good as that of the non-adaptive source classifier. Furthermore, for the discriminant analysis case, its performance is always strictly better. Our target contrastive pessimistic model performs on par with state-of-the-art domain adaptive classifier on sample selection bias settings and outperforms them on more realistic domain adaptation problem settings.

\section*{Acknowledgments} 
This work was supported by the Netherlands Organization for Scientific Research (NWO; grant 612.001.301). 

\section*{Appendix A} \label{proof}
\begin{theorem*}
For a continuous target distribution, with more unique samples than features for every class, $m_k > D$, the empirical target risk of a discriminant analysis model $\hat{R}_{\text{DA}}$ with {\sc tcp} estimated parameters $\hat{\theta}^{\cal T}$ is strictly smaller than the empirical target risk of a discriminant analysis model with parameters $\hat{\theta}^{\cal S}$ estimated on the source domain:
\begin{align}
	\hat{R}_{\text{DA}} \big( \hat{\theta}^{\cal T} \given z,u \big) \ < \ \hat{R}_{\text{DA}} \big( \hat{\theta}^{\cal S} \given z,u \big) \nonumber
\end{align}
\end{theorem*}

\begin{proof}
Let $\{(x_i,y_i)\}_{i=1}^{n}$ and $\{(z_j,u_j)\}_{j}^{m}$ be sample sets of size $n,m$ drawn from continuous distributions $p_{\cal S}$ and $p_{\cal T}$, respectively. Consider a discriminant analysis model parameterized either as $\theta = (\pi_1, \dots, \pi_K, \mu_1, \dots, \mu_K, \Sigma_1, \dots \Sigma_K)$ for {\sc qda} or $\theta = (\pi_1, \dots, \pi_K, \mu_1, \dots, \mu_K, \Sigma)$ for {\sc lda}. $\hat{R}_{\text{DA}}$ denotes empirical risk measured with the Gaussian average negative log-likelihood weighed by a set of soft labels $q$: $m^{-1} \sum_{j=1}^{m} \sum_{k=1}^{K} - q_{kj} \log \mathcal{N}(z_j,k \given \theta_k)$. The sample covariance matrix is required to be non-singular, which is guaranteed when there are more unique samples than features for every class, $m_k > D$. In the {\sc lda} case, $D+K$ unique samples are sufficient. Let $\hat{\theta}^{\cal S}$ be the parameters estimated on labeled source data; $\hat{\theta}^{\cal S} = \underset{\theta \in \Theta}{\arg \min} \ \hat{R}_{\text{DA}} \big( \theta \given x, y \big)$.

Firstly, for fixed $q$, the minimized contrast between the target risk of any parameter $\theta$ and the source parameters $\hat{\theta}^{\cal S}$ is non-positive, because both parameters sets are elements of the same parameter space, $\theta, \hat{\theta}^{\cal S} \in \Theta$:
\begin{align}
	\underset{\theta \in \Theta}{\min} \ \hat{R}_{\text{DA}} \big(\theta \given z,q \big) - \hat{R}_{\text{DA}} \big(\hat{\theta}^{\cal S} \given z,q \big) \leq 0 \nonumber \, .
\end{align}
$\theta$'s that result in a larger target risk than that of $\hat{\theta}^{\cal S}$ are not minimizers of the contrast. The maximum value it can attain is $0$ which occurs when exactly the same parameters are found; $\theta = \hat{\theta}^{\cal S}$. Considering that the contast is non-positive for any labeling $q$, it is also non-positive with respect to the worst-case labeling:
\begin{align}
	\underset{\theta \in \Theta}{\min} \ \underset{q \in \Delta_{K-1}^{m}}{\max} \ \hat{R}_{\text{DA}} \big( \theta \given z,q \big) - \hat{R}_{\text{DA}} \big( \hat{\theta}^{\cal S} \given z,q \big) \leq 0 \label{nonpos} \, .
\end{align}

Secondly, given that the empirical risk with respect to the true labeling is always less than or equal to the empirical risk with the worst-case labeling, $\hat{R}(\theta \given z,u) \leq \underset{q}{\max} \ \hat{R}( \theta \given z,q)$, the target contrastive risk (\ref{contrast}) with the true labeling $u$ is always less than or equal to the target contrastive pessimistic risk (\ref{tcp}):
\begin{align}
	\underset{\theta \in \Theta}{\min} \ \hat{R}_{\text{DA}} \big( \theta \given z,u \big)& - \hat{R}_{\text{DA}} \big( \hat{\theta}^{\cal S} \given z,u \big) \leq \nonumber \\
	& \underset{\theta \in \Theta}{\min} \ \underset{q \in \Delta_{K-1}^{m}}{\max} \hat{R}_{\text{DA}} \big( \theta \given z,q \big) - \hat{R}_{\text{DA}} \big( \hat{\theta}^{\cal S} \given z,q \big)  \label{conpes} \, .
\end{align}

Let ($\hat{\theta}^{\cal T}, q^{*})$ be the minimaximizer of the target contrastive pessimistic risk on the right-handside of (\ref{conpes}).
%: $(\hat{\theta}^{\cal T},q^{*}) = \underset{\theta \in \Theta}{\arg \min} \ \underset{q \in \Delta_{K-1}^{m}}{ \arg \max} \hat{R}_{\text{DA}} (\theta \given z,q) - \hat{R}_{\text{DA}} (\hat{\theta}^{\cal S} \given z,q)$. 
Plugging these estimates in into (\ref{conpes}) produces:
\begin{align}
	\hat{R}_{\text{DA}} \big( \hat{\theta}^{\cal T} \given z,u \big)& - \hat{R}_{\text{DA}} \big( \hat{\theta}^{\cal S} \given z,u \big) \ \leq \  \hat{R}_{\text{DA}} \big( \hat{\theta}^{\cal T} \given z, q^{*}) - \hat{R}_{\text{DA}} \big(\hat{\theta}^{\cal S} \given z, q^{*} \big)  \label{conpes_t} \, .
\end{align}
Combining inequalities \ref{nonpos} and \ref{conpes_t} gives:
\begin{align}
	\hat{R}_{\text{DA}} \big( \hat{\theta}^{\cal T} \given z,u \big) - \hat{R}_{\text{DA}} \big( \hat{\theta}^{\cal S} \given z,u \big) \leq& 0 \nonumber \\
	\hat{R}_{\text{DA}} \big( \hat{\theta}^{\cal T} \given z,u \big) \leq& \ \hat{R}_{\text{DA}} \big( \hat{\theta}^{\cal S} \given z,u \big) \label{crisks} \, .
\end{align}
However, equality of the two risks in \ref{crisks} occurs with probability $0$, which we will show in the following. 

The total mean for the source classifier consists of the weighted combination of the class means, resulting in the overall source sample average: 
\begin{align}
	\mu^{\cal S} =& \sum_{k=1}^{K} \pi^{\cal S}_k \ \mu^{\cal S}_k \nonumber \\
	=& \sum_{k=1}^{K} \frac{\sum_{i}^{n} y_{ki}}{n} \left[ \frac{1}{\sum_{i}^{n} y_{ki}} \sum_{i=1}^{n} y_{ki} x_i \right] \nonumber \\
	=& \frac{1}{n} \sum_{i=1}^{n} x_i \label {source_totmean} \, .
\end{align}

The total mean for the {\sc tcp-da} estimator is similarly defined, resulting in the overall target sample average:
\begin{align}
	\mu^{\cal T} =& \sum_{k=1}^{K} \pi^{\cal T}_{k} \ \mu^{\cal T}_k \nonumber \\
	=& \sum_{k=1}^{K} \frac{\sum_{j}^{m} q_{kj}}{m} \left[ \frac{1}{\sum_{j}^{m} q_{kj}} \sum_{j=1}^{m} q_{kj} z_j \right] \nonumber \\
	=& \sum_{k=1}^{K} \frac{1}{m} \sum_{j=1}^{m} q_{kj} z_j \label{tcp_totmean1} \\
		=& \frac{1}{m} \sum_{j=1}^{m} z_j \label{tcp_totmean2} \, .
\end{align}
Note that since $q$ consists of probabilities, the sum over classes $\sum_{k}^{K} q_{kj}$ in (\ref{tcp_totmean1}) is $1$, for every sample $j$. Equal risks for these parameter sets, $\hat{R}_{\text{DA}} \big( \theta^{\cal T} \given z,u \big) = \hat{R}_{\text{DA}} (\hat{\theta}^{\cal S} \given z,u)$, implies equality of the total means, $\mu^{\cal T}$ = $\mu^{\cal S}$. By Equations \ref{source_totmean} and \ref{tcp_totmean2}, equal total means imply equal sample averages: $m^{-1} \sum_{j=1}^{m} z_j = n^{-1} \sum_{i=1}^{n} x_i$. Drawing two sets of samples with \emph{exactly equal} sample averages constitutes the union of two single events: 
\begin{align}
	\text{Pr} \big[ \ \bar{x} = \ \bar{z} \ \big] = \ &p_{\cal S}\big( \ {\cal S}_1 = x_1, {\cal S}_2 = x_2, \dots,  {\cal S}_n = x_n \given \bar{x} = \bar{z} \ \big) \nonumber \\
	\ &\cup \ p_{\cal T} \big( \ {\cal T}_1 = z_1, {\cal T}_2 = z_2, \dots,  {\cal T}_m = z_m \given \bar{z} = \bar{x} \ \big) \nonumber \, ,
\end{align}
where the bars over the samples $\bar{x},\bar{z}$ denote the sample averages. By definition, single events under continuous distributions have probability $0$. Therefore, a strictly smaller risk occurs almost surely:
\begin{align}
	\hat{R}_{\text{DA}} \big( \hat{\theta}^{\cal T} \given z,u \big) \ < \ \hat{R}_{\text{DA}} \big( \hat{\theta}^{\cal S} \given z,u \big) \nonumber \, .
\end{align}
\end{proof}

\section*{References}
\bibliography{kouw_pr17a}

\end{document}